%% file: main.tex
\def\year{2022}\relax
\documentclass[letterpaper]{article} 
\usepackage{aaai22}  
\usepackage{times}  
\usepackage{helvet}  
\usepackage{courier}  
\usepackage[hyphens]{url}  
\usepackage{graphicx} 
\usepackage{natbib}  
\usepackage{caption} 
\DeclareCaptionStyle{ruled}{labelfont=normalfont,labelsep=colon,strut=off} 
\usepackage{algorithm}
\usepackage{placeins}
\usepackage{algorithmicx}
\usepackage{tabularx}
\usepackage{graphicx}
\usepackage{placeins}
\usepackage{centernot}
\usepackage{siunitx}
\usepackage{xspace}
\usepackage{amsmath}
\usepackage{hyperref}
\usepackage{booktabs}
\usepackage{multirow}
\usepackage{subcaption}
\usepackage{makecell}
\usepackage[noend]{algpseudocode}
\newtheorem{theorem}{Theorem}

\newenvironment{proof}{{\bfseries Proof}}{}

\usepackage{newfloat}
\usepackage{listings}
\floatstyle{ruled}
\newfloat{listing}{tb}{lst}{}
\floatname{listing}{Listing}
\title{ePA*SE: Edge-Based Parallel A* for Slow Evaluations}

\author{
Shohin Mukherjee, Sandip Aine, Maxim Likhachev
}
\affiliations{
    The Robotics Institute, CMU\\
    \{shohinm, asandip, mlikhach\}@andrew.cmu.edu\\
}

\usepackage{bibentry}

\newcommand{\epase}{ePA*SE\xspace}
\newcommand{\eas}{eA*\xspace}
\newcommand{\wepase}{\mbox{w-}\epase}
\newcommand{\wepaseold}{\wepase w/o thread mgt.}
\newcommand{\weas}{w-\eas\xspace}
\newcommand{\state}{\ensuremath{\mathbf{s}}\xspace}
\newcommand{\midstate}{\ensuremath{\state_m}\xspace}
\newcommand{\States}{\ensuremath{\mathcal{S}}\xspace}
\newcommand{\startstate}{\ensuremath{\state_0}\xspace}
\newcommand{\ac}{\ensuremath{\mathbf{a}}\xspace}
\newcommand{\dac}{\ensuremath{\mathbf{a^d}}\xspace}
\newcommand{\pac}{\ensuremath{\mathbf{a^d}}\xspace}
\newcommand{\Aset}{\ensuremath{\mathcal{A}}\xspace}
\newcommand{\goalreg}{\ensuremath{\mathcal{G}}\xspace}
\newcommand{\ed}{\ensuremath{e}\xspace}
\newcommand{\ped}{\ensuremath{e^d}\xspace}
\newcommand{\edge}{\ensuremath{(\mathbf{s},\mathbf{a})}\xspace}
\newcommand{\pedge}{\ensuremath{(\mathbf{s},\mathbf{a^d})}\xspace}
\newcommand{\pedgenext}{\ensuremath{(\mathbf{s}',\mathbf{a^d})}\xspace}
\newcommand{\open}{\ensuremath{\textit{OPEN}}\xspace}
\newcommand{\closed}{\ensuremath{\textit{CLOSED}}\xspace}
\newcommand{\be}{\ensuremath{\textit{BE}}\xspace}
\newcommand{\gval}{\ensuremath{g}}
\newcommand{\gopt}{\gval^*}
\newcommand{\cost}{\ensuremath{c}\xspace}
\newcommand{\costopt}{\ensuremath{\cost^*}\xspace}
\newcommand{\fval}{\ensuremath{f}}
\newcommand{\hval}{\ensuremath{h}}
\newcommand{\wh}{\ensuremath{w}\xspace}
\newcommand{\wi}{\ensuremath{\epsilon}\xspace}
\newcommand{\numexpanded}{\ensuremath{n\_successors\_generated}\xspace}
\newcommand{\graph}{\ensuremath{G}\xspace}
\newcommand{\vertex}{\ensuremath{v}\xspace}
\newcommand{\Vertices}{\ensuremath{\mathcal{V}}\xspace}
\newcommand{\Edges}{\ensuremath{\mathcal{E}}\xspace}
\newcommand{\plan}{\ensuremath{\pi}\xspace}
\newcommand{\numthreads}{\ensuremath{N_t}\xspace}
\newcommand{\bfactor}{\ensuremath{M}\xspace}

\newcommand{\pick}{\textsc{Pick}\xspace}
\newcommand{\place}{\textsc{Place}\xspace}
\newcommand{\swap}{\textsc{SwitchArm}\xspace}

\newcommand{\assumption}{\textbf{Assumption}\xspace}
\newcommand{\contradiction}{\textbf{Contradiction}\xspace}

\newcommand{\tsim}{\ensuremath{t_s}\xspace}
\newcommand{\tplan}{\ensuremath{t_p}\xspace}

\begin{document}

\maketitle

\begin{abstract}
\input{01abstract}
\end{abstract}

\FloatBarrier
\section{Introduction}
\input{02introduction}

\FloatBarrier
\section{Related Work}
\input{03background}

\FloatBarrier
\section{Problem Definition}
\label{sec:problem}
\input{04problem}

\FloatBarrier
\section{Method}
\label{sec:methods}
\input{05methods}

\FloatBarrier
\section{Properties}
\label{sec:Properties}
\input{06properties}

\FloatBarrier
\section{Evaluation}
\input{07evaluation}

\FloatBarrier
\section{Conclusion and Future Work}
\input{08conclusion_future}

\FloatBarrier
\section{Acknowledgements}
\input{09acknowledgements}

\bibliography{main}

\end{document}

%% file: 01abstract.tex
Parallel search algorithms harness the multithreading capability of modern processors to achieve faster planning. One such algorithm is PA*SE (Parallel A* for Slow Expansions), which parallelizes state expansions to achieve faster planning in domains where state expansions are slow. In this work, we propose \epase (Edge-Based Parallel A* for Slow Evaluations) that improves on PA*SE by parallelizing edge evaluations instead of state expansions. This makes \epase more efficient in domains where edge evaluations are expensive and need varying amounts of computational effort, which is often the case in robotics. On the theoretical front, we show that \epase provides rigorous optimality guarantees. In addition, \epase can be trivially extended to handle an inflation weight on the heuristic resulting in a bounded suboptimal algorithm \wepase (Weighted \epase) that trades off optimality for faster planning. On the experimental front, we validate the proposed algorithm in two different planning domains: 1) motion planning for 3D humanoid navigation and 2) task and motion planning for a dual-arm robotic assembly task. We show that \epase can be significantly more efficient than PA*SE and other alternatives. The open-source code for ePA*SE along with the baselines is available here: \\\url{https://github.com/shohinm/parallel_search}

%% file: 02introduction.tex
Graph search algorithms such as A* and its variants~\cite{hart1968formal, pohl1970heuristic, aine2016multi} are widely used in robotics for task and motion planning problems which can be formulated as a shortest path problem on an embedded graph in the state-space of the domain~\cite{kusnur2021planning, mukherjee2021reactive}. A* maintains an open list (priority queue) of discovered states, and at any point in the search, it expands the state with the smallest priority (f-value) in the list. During the expansion, it generates the successors of the state and evaluates the cost of each edge connecting the expanded state to its successors. In robotics applications such as in motion planning, edge evaluation tends to be the bottleneck in the search. For example, in planning for robot-manipulation, edge evaluation typically corresponds to collision-checks of a robot model against the world model at discrete interpolated states on the edge. Depending on how these models are represented (meshes, spheres, etc.) and how finely the edges are sampled for collision-checking, evaluating an edge can get very expensive. As a consequence, the state expansions are typically slow.

In order to speed up planning in domains where state expansions are slow, an optimal parallelized planning algorithm PA*SE (Parallel A* for Slow Expansions) and its (bounded) suboptimal version wPA*SE (Weighted PA*SE) were developed~\cite{phillips2014pa}. Unlike other parallel search algorithms, in which the number of times a state can be re-expanded increases with the degree of parallelization~\cite{irani1986parallel, zhou2015massively, he2021efficient}, PA*SE expands states in a way that each state is expanded at most once. The key idea in PA*SE is that a state \state can be expanded before another state $\state'$ if $\state$ is \textit{independent} of $\state'$ i.e. expansion of $\state'$ cannot lead to a shorter path to $\state$. If the independence relationship holds in both directions i.e. $\state'$ is also independent of $\state$, then $\state$ and $\state'$ can be expanded in parallel. Though PA*SE parallelizes state expansions, for a given state, the successors are generated sequentially. This is not the most efficient strategy, especially for domains with large branching factors. In addition, this strategy is particularly inefficient in domains where there is a large variance in the edge evaluation times. Consider a state being expanded with several outgoing edges, such that the first edge is expensive to evaluate, while the others are relatively inexpensive. In this case, since a single thread is evaluating all of the edges in sequence, the evaluations of the cheap edges will be held up by the one expensive edge. This happens often in planning for robotics. Consider full-body planning for a humanoid. Evaluating a primitive that moves just the wrist joint of the robot requires collision checking of just the wrist. However, evaluating the primitive that moves the base of the robot, requires fully-body collision checking of the entire robot. One way to avoid this would be to evaluate the outgoing edges in parallel, which PA*SE doesn't do. However, this seemingly trivial modification does not solve another cause of inefficiency in PA*SE i.e. the evaluation of the outgoing edges from a given state is tightly coupled with the expansion of the state. In other words, all the outgoing edges from a given state must be evaluated at the same time when the state is expanded. This leads to more edges being evaluated than is necessary, as we will show in our experiments.

Therefore in this work, we develop an improved optimal parallel search algorithm, ePA*SE (Edge-Based Parallel A* for Slow Evaluations), that eliminates these inefficiencies by 1) decoupling edge evaluation from state expansions and 2) parallelizing edge evaluations instead of state expansions. ePA*SE exploits the insight that the root cause of slow expansions is typically slow edge evaluations. Each ePA*SE thread is responsible for evaluating a single edge, instead of expanding a state and evaluating all outgoing edges from it, all in a single thread, like in PA*SE. This makes ePA*SE significantly more efficient than PA*SE and we show this by evaluating it on two planning domains that are quite different: 1) 3D indoor navigation of a mobile manipulator and 2) a task and motion planning problem of stacking a set of blocks by a dual-arm robot.

%% file: 03background.tex
Parallel planning algorithms seek to make planning faster by leveraging parallel processing. 

\subsubsection{Parallel sampling-based algorithms}
There are a number of approaches that parallelize sampling-based planning algorithms. Probabilistic roadmap (PRM) based methods, in particular, can be trivially parallelized, so much so that they have been described as ``embarrassingly parallel"~\cite{amato1999probabilistic}. In these approaches, several parallel processes cooperatively build the roadmap in parallel~\cite{jacobs2012scalable}. Parallelized versions of RRT have also been developed in which multiple cores expand the search tree by sampling and adding multiple new states in parallel~\cite{devaurs2011parallelizing, ichnowski2012parallel, jacobs2013scalable, park2016parallel}. However, in many planning domains involving planning with controllers~\cite{butzke2014state}, sampling of states is typically not possible. One such class of planning domains where state sampling is not possible is simulator-in-the-loop planning, which uses an expensive physics simulator to generate successors~~\cite{liang2021search}. We, therefore, focus on the more general technique of search-based planning which does not rely on state sampling. 

\subsubsection{Parallel search-based algorithms}
A trivial approach to achieve parallelization in weighted A* is to generate successors in parallel when expanding a state. The downside is that this leads to minimal improvement in performance in domains with a low branching factor. Another approach that Parallel A*~\cite{irani1986parallel} takes, is to expand states in parallel while allowing re-expansions to account for the fact that states may get expanded before they have the minimal cost from the start state. This leads to a high number of state expansions. There are a number of other approaches that employ different parallelization strategies~\cite{evett1995massively, zhou2015massively, burns2010best}, but all of them could potentially expand an exponential number of states, especially if they employ a weighted heuristic. In contrast, PA*SE~\cite{phillips2014pa} parallelly expands states at most once, in such a way that does not affect the bounds on the solution quality. Though PA*SE parallelizes state expansions while preventing re-expansions, as explained earlier, it is not efficient in domains where edge evaluations are expensive since each PA*SE thread sequentially evaluates the outgoing edges of a state being expanded. A parallelized lazy planning algorithm, MPLP~\cite{mukherjee2022mplp}, achieves faster planning by running the search and evaluating edges asynchronously in parallel. Just like all lazy search algorithms, MPLP assumes that successor states can be generated without evaluating edges, which allows the algorithm to defer edge evaluations and lazily proceed with the search. However, this assumption doesn't hold true for a number of planning domains in robotics. In particular, consider planning problems that use a high-fidelity physics simulator to evaluate actions involving object-object and object-robot interactions~\cite{liang2021search}. The generation of successor states is typically not possible without a very expensive simulator call. In such domains, where edge evaluation cannot be deferred, MPLP is not applicable. One of the domains in our experiments falls in this class of problems.

\subsubsection{GPU-based parallel algorithms} There has also been work on parallelizing A* search on a single GPU~\cite{zhou2015massively} or multiple GPUs~\cite{he2021efficient} by utilizing multiple parallel priority queues. Since these approaches must allow state re-expansions, the number of expansions increases exponentially with the degree of parallelization. In addition, GPU-based parallel algorithms have a more fundamental limitation which stems from the single-instruction-multiple-data (SIMD) execution model of a GPU. This means that a GPU can only run the same set of instructions on multiple data concurrently. This severely limits the design of planning algorithms in several ways. Firstly, the code for expanding a state must be identical, irrespective of what state is being expanded. Secondly, the set of states must be expanded in a batch. This is problematic in domains that have complex actions that correspond to forward simulating dissimilar controllers. In contrast to these approaches, ePA*SE achieves parallelization of edge evaluations on the CPU which has a multiple-instruction-multiple-data (MIMD) execution model. This allows ePA*SE the flexibility to efficiently parallelize dissimilar edges, and therefore generalize across all types of planning domains.

%% file: 04problem.tex
Let a finite graph $\graph = (\Vertices, \Edges)$ be defined as a set of vertices \Vertices and directed edges \Edges. Each vertex $\vertex \in \Vertices$ represents a state \state in the state space of the domain \States. An edge $\ed \in \Edges$ connecting two vertices $\vertex_1$ and $\vertex_2$ in the graph represents an action $\ac \in \Aset$ that takes the agent from corresponding states $\state_1$ to $\state_2$. In this work, we assume that all actions are deterministic. Hence an edge \ed can be represented as a pair \edge, where \state is the state at which action \ac is executed. For an edge \ed, we will refer to the corresponding state and action as $\ed.\state$ and $\ed.\ac$ respectively. In addition, we will use the following notations:

\begin{itemize}
    \item \startstate is the start state and \goalreg is the goal region.
    \item $\cost:\Edges \rightarrow [0,\infty]$ is the cost associated with an edge.
    \item $\gval(\state)$ or g-value is the cost of the best path to \state from \startstate found by the algorithm so far.
    \item $\hval(\state)$ is a consistent and therefore admissible heuristic~\cite{russell2010artificial}. It never overestimates the cost to the goal.
\end{itemize}

A path \plan is defined by an ordered sequence of edges $\ed_{i=1}^N = \edge_{i=1}^N$, the cost of which is denoted as 
$\cost(\plan) = \sum_{i=1}^N \cost(\ed_i)$.
The objective is to find a path \plan from $\state_0$ to a state in the goal region \goalreg with the optimal cost \costopt. There is a computational budget of \numthreads threads available, which can run in parallel. Similar to PA*SE, we assume there exists a pairwise heurisitic function $\hval(\state, \state')$ that provides an estimate of the cost between any pair of states. It is forward-backward consistent i.e. $\hval(\state, \state'') \leq \hval(\state, \state') + \hval(\state', \state'')~\forall~\state,\state',\state''$ and $\hval(\state, \state') \leq \costopt(\state, \state')~\forall~\state, \state'$. Note that using \hval~for both the unary heuristic $\hval(\state)$ and the pairwise heuristic $\hval(\state, \state')$ is a slight abuse of notation, since these are different functions.

%% file: 05methods.tex
\begin{figure*}[ht]
    \centering
    \includegraphics[width=\textwidth]{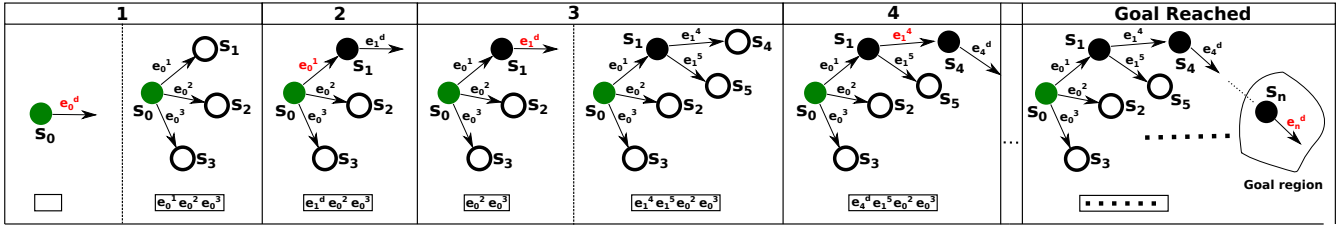}
    \caption{Example of \eas: (1) The dummy edge $\ed_0^d$ originating from \startstate is expanded and the real edges $[\ed_0^1, \ed_0^2, \ed_0^3]$ are inserted into \open. (2) $\ed_0^1$ is expanded, during which it is evaluated and the successor $\state_1$ is generated. A dummy edge $\ed_1^d$ from $\state_1$ is inserted into \open. (3) $\ed_1^d$ is expanded and the real edges $[\ed_1^4,\ed_1^5]$ are inserted into \open. (4) $\ed_1^4$ is expanded, during which it is evaluated and the successor $\state_4$ is generated and a dummy edge $\ed_4^d$ is inserted into open. This goes on until a dummy edge $\ped_n$ is expanded whose source state belongs to the goal region i.e. $\state_n \in \goalreg$.}
    \label{epase/fig/overview}
\end{figure*}

\epase leverages the key algorithmic contribution of PA*SE i.e. parallel expansions of independent states but instead uses it to parallelize edge evaluations. In doing so, \epase further improves the efficiency of PA*SE in domains with expensive to evaluate edges. \epase obeys the same invariant as A* and PA*SE that when a state is expanded, its g-value is optimal. Therefore, every state is expanded at most once. However, unlike in PA*SE, where each thread is responsible for expanding a single state at a time, each \epase thread is responsible for evaluating a single edge at a time. In order to build up to \epase, we first describe a serial version of the proposed algorithm \eas (Edge-based A*). We then explain how \eas can be parallelized to get to \epase, using the key idea behind PA*SE.

\subsubsection{\eas} 
\label{epase/sec/method/overview}
The first key algorithmic difference in \eas as compared to A* is that the open list \open contains edges instead of states. We introduce the term \textit{expansion of an edge} and explicitly differentiate it from the expansion of a state. In A*, during the expansion of a state, all its successors are generated and, unless they have already been expanded, are either inserted into the open list or repositioned with the updated priority. In \eas, expansion of an edge \edge involves evaluating the edge to generate the successor state $\state'$ and adding/updating (but not evaluating) the edges originating from $\state'$ into \open with the same priority of $\gval(\state') + \hval(\state')$. This choice of priority ensures that the edges originating from states that would have the same (state-) expansion priority in A* have the same (edge-) expansion priority in \eas. A state is defined as \textit{partially expanded} if at least one (but not all) of its outgoing edges has been expanded or is under expansion, while it is defined as  \textit{expanded} if all its outgoing edges have been expanded. \eas uses the following data structures as the key ingredients of the algorithm.
\begin{itemize}
    \item \open: A priority queue of edges (not states) that the search has generated but not expanded, where the edge with the smallest key/priority is placed in the front of the queue. The priority of an edge $\ed=\edge$ in \open is
    $\fval\left(\edge\right) = \gval(\state) + \hval(\state)$.
    \item \be: The set of states that are partially expanded.
    \item \closed: The set of states that have been expanded.
\end{itemize}

Naively storing edges instead of states in \open introduces an inefficiency. In A*, the g-value of a state \state can change many times during the search until the state is expanded at which point it is added to \closed. Every time this happens, \open has to be rebalanced to reposition \state. In \eas, every time $\gval(\state)$ changes, the position of all of the outgoing edges from \state need to be updated in \open. This increases the number of times \open has to be rebalanced, which is an expensive operation. However, since the edges originating from \state have the same priority i.e.  $\gval(\state) + \hval(\state)$, this can be avoided by replacing all the outgoing edges from \state by a single \textit{dummy} edge $\ped=\pedge$, where \pac stands for a dummy action. The dummy edge stands as a placeholder for all the \textit{real} edges originating from \state. Every time $\gval(\state)$ changes, only the dummy edge has to be repositioned. Unlike what happens when a real edge is expanded, when the dummy edge \pedge is expanded, it is replaced by the outgoing real edges from \state in \open. When a state's dummy edge is expanded or is under expansion, it is also considered to be partially expanded and is therefore added to \be. When all the outgoing real edges of a state have been expanded, it is moved from \be to \closed. The g-value $\gval(\state)$ of a state \state in either \be or \closed can no longer change and hence the real edges originating from \state will never have to be updated in \open.  

\eas can be trivially extended to handle an inflation factor on the heuristic like wA* which leads to a more goal-directed search \weas (Weighted \eas). Fig.~\ref{epase/fig/overview} show an example of \weas in action. Let $\ed_i^j$ refer to an edge from state $\state_i$ to $\state_j$ and $\ped_i$ refer to a dummy edge from $\state_i$. The states that are generated are shown in solid circles. The hollow circles represent states that are not generated and hence the incoming edges to these states are not evaluated. During the first expansion, the dummy edge $\ped_0$ originating from \startstate is expanded and the real edges $[\ed_0^1, \ed_0^2, \ed_0^3]$ are inserted into \open. In the second expansion, the edge $\ed_0^1$ is expanded, during which it is evaluated and the successor $\state_1$ is generated. A dummy edge ($\ped_1$) from $\state_1$ is inserted into \open. In the third expansion, $\ed_1^d$ is expanded and the real edges $[\ed_1^4,\ed_1^5]$ are inserted into \open. In the fourth expansion, the edge $\ed_1^4$ is expanded, during which it is evaluated and the successor $\state_4$ is generated and a dummy edge $\ped_4$ is inserted into open. This goes on until a dummy edge $\ped_n$ is expanded whose source state belongs to the goal region i.e. $\state_n \in \goalreg$. 

If the heuristic is informative, \weas evaluates fewer edges than wA*. In the example shown in Fig.~\ref{epase/fig/overview}, the edges $[\ed_0^2, \ed_0^3, \ed_1^5]$ do not get evaluated. Since wA* evaluates all outgoing edges of an expanded state, these edges would be evaluated in the case of wA* (with the same heuristic and inflation factor) when their source states are expanded ($\state_0$ and $\state_1$). Additionally, similar to how wPA*SE parallelizes wA*, \weas can be parallelized to obtain a highly efficient algorithm \wepase. Since \wepase is a trivial extension of \epase, we instead describe how \eas can be parallelized to obtain \epase.

\subsubsection{\eas to \epase}
\eas can be parallelized using the key idea behind PA*SE i.e. parallel expansion of independent states, and applying it to edge expansions, resulting in \epase. \epase has two key differences from PA*SE that makes it more efficient:

\begin{enumerate}
    \item Evaluation of edges is decoupled from the expansion of the source state giving the search the flexibility to figure out what edges need to be evaluated.
    \item Evaluation of edges is parallelized.
\end{enumerate}

In addition to \open and \closed, PA*SE uses another data structure \be (Being Expanded) to store the set of states currently being expanded by one of the threads. It uses a pairwise independence check on states in the open list to find states that are safe to expand in parallel. A state \state is safe to expand if $\gval(\state)$ is already optimal. In other words, there is no other state that is currently being expanded (in \be), nor in \open that can reduce $\gval(\state)$. Formally, a state \state is defined to be independent of state \state' iff

\begin{equation}
    \label{eq:state_independence}
    \gval(\state) - \gval(\state') \leq \hval(\state',\state)
\end{equation}

It can be proved that \state is independent of states in \open that have a larger priority than \state~\cite{phillips2014pa}. However, the independence check has to be performed against the states in \open with a smaller priority than \state as well as the states that are in \be.

Like in \eas, \be in \epase stores the states that are partially expanded, as per the definition of partial expansion in \eas. Since \epase stores edges in \open instead of states and each \epase thread expands edges instead of states, the independence check has to be modified. An edge \ed is safe to expand if Equations~\ref{eq:ind_check_1}~and~\ref{eq:ind_check_2} hold.

\begin{align}
    \label{eq:ind_check_1}
    \begin{split}
        \gval(\ed.\state) - \gval(\ed'.\state) \leq \hval(\ed'.\state, \ed.\state)\\
        \forall\ed' \in  \open~|~\fval\left(\ed'\right) < \fval\left(\ed\right)
    \end{split}
\end{align}

\begin{align}
    \label{eq:ind_check_2}
    \begin{split}
        \gval(\ed.\state) - \gval(\state') \leq \hval(\state', \ed.\state)~\forall\state' \in \be
    \end{split}
\end{align}

Equation~\ref{eq:ind_check_1} ensures that there is no edge in \open with a priority smaller than that of \ed, that upon expansion, can lower the g-value of $\ed.\state$ and hence lower the priority of \ed. In other words, the source state \state of edge \ed is independent of the source states of all edges in \open which have a smaller priority than \ed. Equation~\ref{eq:ind_check_2} ensures that there is no partially expanded state which can lower the g-value of $\ed.\state$. In other words, the source state \state of edge \ed is independent of all states in \be.

\subsubsection{Details}
\input{epase.tex}

The pseudocode for \epase is presented in Alg.~\ref{alg:epase}. The main planning loop in \textsc{Plan} runs on a single thread (thread $0$), and in Line~\ref{alg:epase/open_pop}, an edge is removed for expansion from \open that has the smallest possible  priority and is also safe to expand, as per Equations~\ref{eq:ind_check_1}~and~\ref{eq:ind_check_2}. If such an edge is not found, the thread waits for either \open or \be to change in Line~\ref{alg:epase/wait}. If a safe to expand edge is found, such that the source state of the edge belongs to the goal region, the solution path is returned by backtracking from the state to the start state using back pointers (like in A*) in Line~\ref{alg:epase/construct_path}. Otherwise, the edge is expanded assigned to an edge expansion thread (thread $i=1:\numthreads$) in Line~\ref{alg:epase/assign_edge}. The edge expansion threads are spawned as and when needed to avoid the overhead of running unused threads (Line~\ref{alg:epase/spawn}). The search terminates when either a solution is found, or when \open is empty and all threads are idle (\be is empty), in which case there is no solution.

If the edge to be expanded is a dummy edge, the source \state of the edge is marked as partially expanded by adding it to \be (Line~\ref{alg:epase/add_be}). The real edges originating from \state are added to \open with the same priority as that of the dummy edge i.e. $\gval(\state) + \hval(\state)$.  If the expanded edge is not a dummy edge, it is evaluated (Line~\ref{alg:epase/evaluate}) to obtain the successor $\state'$ and the edge cost $\cost\left(\edge\right)$. This is the expensive operation that \epase seeks to parallelize, which is why it happens lock-free. If the expanded edge reduces $\gval(\state')$, the dummy edge originating from $\state'$ is added/updated in \open. A counter \numexpanded keeps track of the number of outgoing edges that have been expanded for every state. Once all the outgoing edges for a state have been expanded, and hence the state has been expanded, it is removed from \be and added to \closed (Lines~\ref{alg:epase/remove_be} and \ref{alg:epase/add_closed}).

\subsubsection{Thread management}

In PA*SE, the state expansion threads are spawned at the start and each of them independently pulls out states from the open list to expand. When the number of threads is higher than the number of independent states available for expansion at any point in time, the operating system has an unnecessary overhead of spinning unused threads. This causes the overall performance to go down as the number of unused threads goes up (see Fig.~6 in ~\cite{phillips2014pa}). Our initial experiments showed that using a similar thread management strategy in ePA*SE leads to a similar degradation in performance as the number of threads is increased beyond the optimal number of threads, even though the peak performance of ePA*SE is substantially higher than that of PA*SE. In order to prevent this degradation in performance, ePA*SE employs a different thread management strategy. There is a single thread that pulls out edges from the open list and it spawns edge expansion threads as needed but capped at \numthreads (Line~\ref{alg:epase/spawn}). When \numthreads is higher than the number of independent edges available for expansion at any point in time, only a subset of available threads get spawned preventing performance degradation, as we will show in our experiments.

\subsubsection{\wepase}

\wepase is a bounded suboptimal variant of \epase that trades off optimality for faster planning. Similar to wPA*SE, \wepase introduces two inflation factors, the first of which, $\wi\geq1$, relaxes the independence rule (Equations~\ref{eq:ind_check_1}~and~\ref{eq:ind_check_2}) as follows.

\begin{align}
    \label{eq:ind_check_3}
    \begin{split}
        \gval(\ed.\state) - \gval(\ed'.\state) \leq \wi\hval(\ed'.\state, \ed.\state)\\
        \forall\ed' \in  \open~|~\fval\left(\ed'\right) < \fval\left(\ed\right)
    \end{split}
\end{align}

\begin{align}
    \label{eq:ind_check_4}
    \begin{split}
        \gval(\ed.\state) - \gval(\state') \leq \wi\hval(\state', \ed.\state)~\forall\state' \in \be
    \end{split}
\end{align}

The second factor $\wh\geq1$ is used to inflate the heuristic in the priority of edges in \open i.e. $\fval\left(\edge\right)=\gval(\state)~+~\wh~\cdot~\hval(\state)$ which makes the search more goal directed. As long as $\wi\geq\wh$, the solution cost is bounded by $\wi\cdot\costopt$ (Theorem~\ref{th:suboptmality}). Note that \wh can be greater than \wi, but then  Equation~\ref{eq:ind_check_3} has to consider source states of all edges in \open and the solution cost will be bounded by $\wh\cdot\costopt$ (Theorem~\ref{th:1}). Since this leads to significantly more independence checks, the $\wi\geq\wh$ relationship is typically recommended in practice.

%% file: epase.tex
\begin{algorithm}[]
\caption{\label{alg:epase} ePA*SE}
\begin{footnotesize}
\begin{algorithmic}[1]
\State $\Aset\gets\text{ action space }$, $\numthreads \gets$ number of threads, $\graph \gets \emptyset$
\State $\startstate\gets\text{ start state }$, $\goalreg\gets\text{ goal region}$, $terminate \gets \text{False}$
\Procedure{Plan}{}
    \State $\forall\state\in\graph$,~$\state.\gval\gets\infty$,~$\numexpanded(s)=0$
    \State $\state_0.\gval\gets0$
    \State insert $(\startstate, \dac)$ in \open
    \Comment{Dummy edge from \startstate}
    \State LOCK
    \While{$\textbf{not } terminate$}
        \If{$\open=\emptyset\textbf{ and }\be=\emptyset$}
            \State $terminate = \text{True}$
            \State UNLOCK
            \State $\Return~\emptyset$
        \EndIf
        \State remove an edge \edge from \open that has the \newline\hspace*{2.9em} smallest $\fval(\edge)$ among all states in \open that \newline\hspace*{2.9em} satisfy Equations \ref{eq:ind_check_1} and \ref{eq:ind_check_2}
        \label{alg:epase/open_pop}
        \If{such an edge does not exist}
            \State UNLOCK
            \State wait until \open or \be change
            \label{alg:epase/wait}
            \State LOCK
            \State continue
        \EndIf
        \If{$\state \in \goalreg$}
            \State $terminate = \text{True}$
            \State UNLOCK
            \State $\Return~\textsc{Backtrack(\state)}$
            \label{alg:epase/construct_path}
        \Else
            \State UNLOCK
            \While{\edge has not been assigned a thread}
                \For{$i=1:\numthreads$}
                    \If{thread $i$ is available}
                        \If{thread $i$ has not been spawned}
                            \State Spawn $\textsc{EdgeExpandThread}(i)$
                            \label{alg:epase/spawn}
                        \EndIf
                        \State Assign \edge to thread $i$
                        \label{alg:epase/assign_edge}
                    \EndIf
                \EndFor
            \EndWhile
            \State LOCK
        \EndIf
    \EndWhile 
    \State $terminate = \text{True}$
    \State UNLOCK
\EndProcedure
\Procedure{EdgeExpandThread}{$i$}
    \While{$\textbf{not } terminate$}
        \If{thread $i$ has been assigned an edge \edge}
            \State $\textsc{Expand}\left(\edge\right)$
        \EndIf
    \EndWhile
\EndProcedure
\Procedure{Expand}{$\edge$}
    \State LOCK
    \If{$\ac = \pac$}
        \State insert \state in \be
        \label{alg:epase/add_be}
        \For{$\ac \in \Aset$}
            \State $\fval\left(\edge\right) = \gval(\state) + \hval(\state)$
            \State insert \edge in \open with $\fval\left(\edge\right)$
        \EndFor
    \Else
        \State UNLOCK
        \State $\state', \cost\left(\edge\right)  \gets \textsc{GenerateSuccessor}
        \left(\edge\right)$
        \label{alg:epase/evaluate}
        \State LOCK
        \If{$\state' \notin \closed\cup\be$~and\\ ~~~~~~~~~~~~$\gval(\state')>\gval(\state)+\cost\left(\edge\right)$}
            \State $\gval(\state') = \gval(\state) + \cost\left(\edge\right)$
            \State $\state'.parent = \state$ 
            \State $\fval\left(\pedgenext\right) = \gval(\state') + \hval(\state')$
            \State insert/update \pedgenext in \open with $\fval\left(\pedgenext\right)$
        \EndIf
        \State $\numexpanded(\state)+=1$
        \If{$\numexpanded(\state) = |\Aset|$}
            \State remove \state from \be
            \label{alg:epase/remove_be}
            \State insert \state in \closed
            \label{alg:epase/add_closed}
        \EndIf
    \EndIf
    \State UNLOCK
\EndProcedure
\end{algorithmic}
\end{footnotesize}
\end{algorithm}

%% file: 06properties.tex
\wepase has identical properties to that of wPA*SE~\cite{phillips2014pa} and can be proved similarly with minor modifications.

\begin{theorem}[\textbf{Bounded suboptimal expansions}]
\label{th:1}
When \wepase that performs independence checks against all states in \be and source states of \textbf{all} edges in \open, chooses an edge \ed for expansion, then $\gval(\ed.\state) \leq \lambda\gopt(\state)$, where $\lambda = \max(\wi, \wh)$.
\end{theorem}
\begin{proof}
Assume, for the sake of contradiction, that $\gval(\ed.\state)>\lambda\gopt(\ed.\state)$ directly before edge \ed is expanded, and without loss of generality, that $\gval(\ed'.\state)\leq\lambda\gopt(\ed'.\state)$ for all edges $\ed'$ selected for expansion before \ed (\assumption). Let $\state = \ed.\state$ for ease of notation. Consider any 
cost-minimal path $\plan(\startstate, \state)$ from \startstate to \state. Let $\midstate$ be the closest state to \startstate on $\plan(\startstate, \state)$ such that either 1) there exists at least one edge in $\open$ with source state $\midstate$ or 2) $\midstate$ is in \be. $\midstate$ is no farther away from \startstate on $\plan(\startstate, \state)$ than \state since \ed is in \open. Therefore, let $\plan(\startstate, \midstate)$ and $\plan(\midstate, \state)$ be the subpaths of $\plan(\startstate, \state)$ from \startstate to $\midstate$  and from $\midstate$ to \state, respectively.

If $\midstate=\startstate$, then $\gval(\midstate)\leq\lambda\gopt(\midstate)$ since $\gval(\startstate)=\gopt(\startstate)=0$ ($\contradiction~\pmb{1}$). Otherwise, let $\state_p$ be the predecessor of \midstate on $\plan(\startstate, \state)$. $\state_p$ has been expanded (i.e. all edges outgoing edges of $\state_p$ have been expanded) since 
every state closer to \startstate on $\plan(\startstate, \state)$ than \midstate has been expanded 
(since every unexpanded state on $\plan(\startstate, \state)$ different from \startstate is either in \be or has an outgoing edge in \open, or has a state closer to \startstate on $\plan(\startstate, \state)$ that is either in \be or has an outgoing edge in \open). 
Therefore, since all outgoing edges from $\state_p$ have been expanded, $\gval(\state_p)\leq\lambda\gopt(\state_p)$ because of \assumption. Then, because of the \gval~update of \midstate when the edge from $\state_p$ to \midstate was expanded,
\begin{align}
    \label{eq:4}
    \gval(\midstate) &\leq \gval(\state_p) + \cost(\state_p, \midstate)\nonumber\\
    &\leq\lambda\gopt(\state_p) + \cost(\state_p, \midstate)
\end{align}

Since $\state_p$ is the predecessor of \midstate on the cost-minimal path $\plan(\startstate, \state)$, 
\begin{align}
    \label{eq:5}
    &\gopt(\midstate) = \gopt(\state_p) + \cost(\state_p, \midstate)\nonumber\\
    \implies&\gopt(\state_p) = \gopt(\midstate)- \cost(\state_p, \midstate)
\end{align}

Substituting $\gopt(\state_p)$ from Equation~\ref{eq:5} into Equation~\ref{eq:4},
\begin{align}
    \label{eq:1/1}
    \implies&\gval(\midstate) \leq \lambda\gopt(\midstate) - (\lambda-1)\cost(\state_p, \midstate)\nonumber\\
    \implies&\gval(\midstate) \leq \lambda\gopt(\midstate)\nonumber\\
    \implies& \lambda\cost(\plan(\startstate, \midstate)) = \lambda\gval^*(\midstate) \geq \gval(\midstate)
\end{align}

Since  $\hval(\midstate, \state)$ satisfies forward-backward consistency and is therefore admissible, $\hval(\midstate, \state) \leq \cost(\plan(\midstate, \state))$. Since, $\lambda=\max(\wi,\wh)$, $\wi \leq\lambda$, 
\begin{align}
    \label{eq:1/2}
    \lambda\cost(\plan(\midstate, \state)) \geq \lambda\hval(\midstate, \state) \geq \wi\hval(\midstate, \state)
\end{align}

Adding \ref{eq:1/1} and \ref{eq:1/2},
\begin{align}
    \label{eq:2}
    \lambda\cost(\startstate,\state) &= \lambda\cost(\startstate,\midstate) + \lambda\cost(\midstate,\state)\nonumber\\
    &\geq \gval(\midstate) + \wi\hval(\midstate, \state)
\end{align}

Assuming \wepase performs independence checks against states in \be and source states of all edges in \open when choosing an edge $\ed$ with source $\ed.\state$ to expand, and \midstate is either in \be or there exists atleast one edge with source \midstate in \open,
\begin{align}
\label{eq:3}
    &\wi\hval(\ed'.\state, \ed.\state) \geq \gval(\ed.\state) - \gval(\ed'.\state)\nonumber\\
    &\forall \ed' \in \open~|~\ed'.\state = \midstate\nonumber\\
    \implies&\gval(\midstate) + \wi\hval(\midstate,\state) \geq \gval(\state)
\end{align}

Therefore,
\begin{align*}
    \lambda\gopt(\state) =& \lambda\cost(\plan(\startstate, \state))\\
    \geq&\gval(\midstate) + \wi\hval(\midstate,\state)\tag*{(Using Eq.~\ref{eq:2})}\\
    \geq&\gval(\state)\tag*{(Using Eq.~\ref{eq:3})}\\\\
    &\text{($\contradiction~\pmb{2}$)}
\end{align*}

\contradiction~\textbf{1} and \contradiction~\textbf{2} invalidate the \assumption, which proves \textbf{Theorem~\ref{th:1}}.
\end{proof}

\begin{theorem}
\label{th:2}
If $\wh\leq\wi$, and considering any two edges \ed and \ed' in \open, the source state of \ed is independent of the source state of \ed' if $\fval(\ed)\leq\fval(\ed')$.
\end{theorem}
\begin{proof}
\begin{align*}
    \fval(\ed) &\leq \fval(\ed')\\
    \implies \gval(\ed.\state) + \wh\hval(\ed.\state) &\leq \gval(\ed'.\state) + \wh\hval(\ed'.\state)\\
    \implies \gval(\ed.\state) &\leq \gval(\ed'.\state) + \wh(\hval(\ed'.\state)-\hval(\ed.\state))\\
    &\leq \gval(\ed'.\state) +\wh\hval(\ed'.\state, \ed.\state)\\
    &\text{(forward-backward consistency)}\\
    &\leq \gval(\ed'.\state) + \wi\hval(\ed'.\state, \ed.\state)\\
    &\text{(since $\wh\leq\wi$)}
\end{align*}

Therefore, $\ed.\state$ is independent of $\ed'.\state$ by definition~(Eq.~\ref{eq:state_independence}).
\end{proof}

\begin{theorem}[Bounded suboptimality]
\label{th:suboptmality}
If $\wh\leq\wi$, and \wepase chooses a dummy edge $\ped=\pedge$ for expansion, such that the source state $\state$ belongs to the goal region i.e. $\state \in \goalreg$, then $\gval(\state)\leq\wi\gopt(\state)=\wi\cdot\costopt$.
\end{theorem}
\begin{proof}
This directly follows from Theorems~\ref{th:1}~and~\ref{th:2}.
\end{proof}

\begin{theorem}[Completeness]
\label{th:completeness}
If there exists at least one path \plan in \graph from $\state_0$ to \goalreg, \wepase will find it.
\end{theorem}
\begin{proof}
This proof makes use of Theorem~\ref{th:suboptmality} and is similar to the equivalent proof of serial wA*.
\end{proof}

%% file: 07evaluation.tex
We evaluate \wepase in two planning domains where edge evaluation is expensive. All experiments were carried out on Amazon Web Services (AWS) instances. All algorithms were implemented in C++.

\subsection{3D Navigation}

The first domain is motion planning for 3D ($x,y,\theta$) navigation of a PR2, which is a human-scale dual-arm mobile manipulator robot, in an indoor environment similar to the one used in \cite{narayanan2017heuristic} and shown in Fig.~\ref{epase/fig/nav3d_problem}. Here $x,y$ are the planar coordinates and $\theta$ is the orientation of the robot. The robot can move along 18 simple motion primitives that independently change the three state coordinates by incremental amounts. Evaluating each primitive involves collision checking of the robot model (approximated as spheres) against the world model (represented as a 3D voxel grid) at interpolated states on the primitive. Though approximating the robot with spheres instead of meshes speeds up collision checking, it is still the most expensive component of the search. The computational cost of edge evaluation increases with an increasing granularity of interpolated states at which collision checking is carried out. For our experiments, collision checking is carried out at interpolated states $1$ \si{\cm} apart. The search uses Euclidean distance as the admissible heuristic. We evaluate on 50 trials in each of which the start configuration of the robot and goal region are sampled randomly. We compare \wepase with other CPU-based parallel search baselines. The first baseline is a variant of weighted A* in which during a state expansion, the successors of the state are generated and the corresponding edges are evaluated in parallel. For lack of a better term, we call this baseline Parallel Weighted A* (PwA*). Note that this is very different from the Parallel A* (PA*) algorithm~\cite{irani1986parallel} which has already been shown to underperform wPA*SE~\cite{phillips2014pa}. The second baseline is wPA*SE. These two baselines leverage parallelization differently. PwA* parallelizes the generation of successors, whereas wPA*SE parallelizes state expansions. \wepase on the other hand parallelizes edge evaluations. We also compare against a variation of \wepase (\wepaseold) that uses the thread management strategy of wPA*SE as opposed to the improved thread management strategy described in the Method. Speedup over wA* is defined as the ratio of the average runtime of wA* over the average runtime of a specific algorithm.

\begin{figure}[]
    \centering
    \includegraphics[width=\columnwidth]{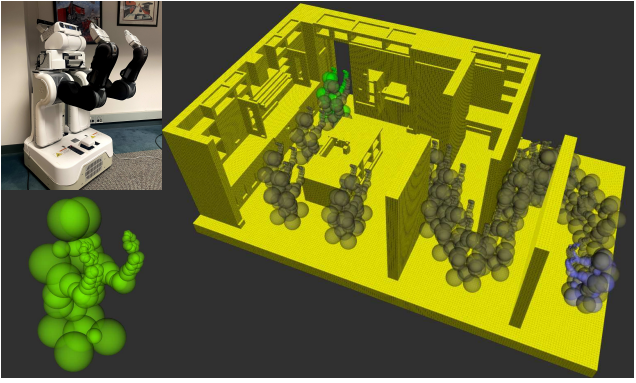}
    \caption{(Navigation) Left: The PR2's collision model is approximated with spheres. 
    Right: The task is to navigate in an indoor map from a given start (purple) and goal (green) states using a set of motion primitives. States at the end of every primitive in the generated plan are shown in black.}
    \label{epase/fig/nav3d_problem}
\end{figure}

\begin{figure}[]
    \centering
    \includegraphics[width=\columnwidth]{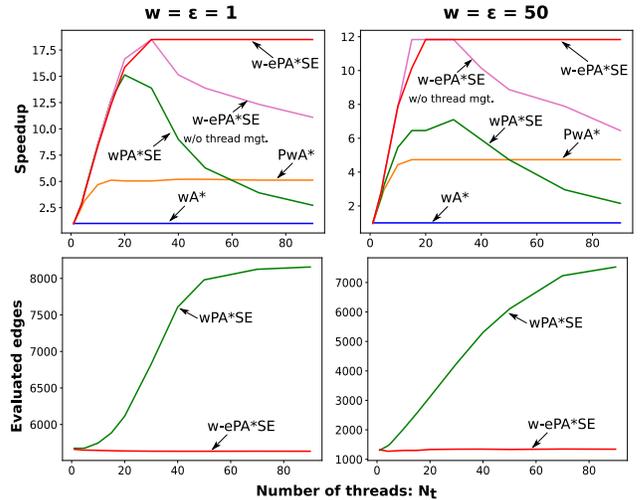}
    \caption{(Navigation) Top: Average speedup achieved by PwA*, wPA*SE and \wepase over wA*. Bottom: Number of edges evaluated by wPA*SE and \wepase.}
    \label{epase/fig/nav3d_plots}
\end{figure}

\begin{table*}[]
\footnotesize
\begin{subtable}{\textwidth}
\centering
\begin{tabular}{ccccccccccccc}
\toprule
\numthreads          & 1    & 4    & 5    & 10   & 15   & 20   & 30   & 40   & 50   & 70   & 90   \\\midrule
wA*       & 3.33  & -    & -    & -    & -    & -    & -    & -    & -    & -    & -    &  \multirow{5}{*}{$\wh = \wi = 1\ \ $}\\
PwA*      & 3.37  & 1.31    & 1.06    & 0.71    & 0.65    & 0.66    & 0.66    & 0.64    & 0.64    & 0.65    & 0.65  \\
wPA*SE     & 3.37 & 1.14 & 0.85 & 0.39 & 0.26 & 0.22 & 0.24 & 0.37 & 0.53 & 0.85 & 1.22 \\
\makecell{\wepase w/o thread mgt.} & 3.43 & 1.17 & 0.88 & 0.39 & 0.26 & 0.20 & 0.18 & 0.22 & 0.24 & 0.27 & 0.30 \\
\wepase      & \textbf{3.34}    & \textbf{1.17} & \textbf{0.87} & \textbf{0.40} & \textbf{0.27} & \textbf{0.21} & \textbf{0.18} & \textbf{0.18} & \textbf{0.18} & \textbf{0.18} & \textbf{0.18} \\
\midrule
\end{tabular}
\end{subtable}
\begin{subtable}{\textwidth}
\centering
\begin{tabular}{ccccccccccccc}
wA*       & 0.71  & -    & -    & -    & -    & -    & -    & -    & -    & -    & -  &  \multirow{5}{*}{$\wh = \wi = 50$}  \\
PwA*      & 0.72  & 0.29    & 0.24    & 0.16    & 0.15    & 0.15    & 0.15    & 0.15    & 0.15    & 0.15    & 0.15    \\
wPA*SE     & 0.71 & 0.28 & 0.22 & 0.13 & 0.11 & 0.11 & 0.10 & 0.12 & 0.15 & 0.24 & 0.33 \\
\makecell{\wepase w/o thread mgt.} & 0.75    & 0.25 & 0.19 & 0.09 & 0.06 & 0.06 & 0.06 & 0.07 & 0.08 & 0.09 & 0.11 \\
\wepase   & \textbf{0.72}    & \textbf{0.25} & \textbf{0.19} & \textbf{0.09} & \textbf{0.07} & \textbf{0.06} & \textbf{0.06} & \textbf{0.06} & \textbf{0.06} & \textbf{0.06} & \textbf{0.06} \\
\midrule
\end{tabular}
\end{subtable}
\caption{(Navigation) Average planning times (\si{\second}) for wA*, PwA*, wPA*SE and \wepase for varying \numthreads, with $\wh=\wi=1$ (top) and with $\wh=\wi=50$ (bottom).}
\label{epase/tab/nav3d_epase_times}
\end{table*}

\begin{table*}[]
\footnotesize
\begin{subtable}{\textwidth}
\centering
\begin{tabular}{ccccccccccccc}
\toprule
\numthreads          & 1    & 4    & 5    & 10   & 15   & 20   & 30   & 40   & 50   & 70   & 90   \\\midrule
wPA*SE     & 5674 & 5673 & 5676 & 5746 & 5885 & 6112 & 6826 & 7607 & 7980 & 8125 & 8156 &  \multirow{2}{*}{$\wh = \wi = 1\ \ $}\\
\wepase   & \textbf{5660} & \textbf{5650} & \textbf{5649} & \textbf{5645} & \textbf{5640} & \textbf{5637} & \textbf{5634} & \textbf{5633} & \textbf{5633} & \textbf{5634} & \textbf{5633} \\
\midrule
\end{tabular}
\end{subtable}
\begin{subtable}{\textwidth}
\centering
\begin{tabular}{ccccccccccccc}
wPA*SE     & 1309 & 1451 & 1526 & 2028 & 2561 & 3121 & 4251 & 5307 & 6105 & 7231 & 7526 &  \multirow{2}{*}{$\wh = \wi = 50$}\\
\wepase   & \textbf{1324} & \textbf{1273} & \textbf{1277} & \textbf{1300} & \textbf{1301} & \textbf{1333} & \textbf{1343} & \textbf{1345} & \textbf{1334} & \textbf{1348} & \textbf{1343} \\
\midrule
\end{tabular}
\end{subtable}
\caption{(Navigation) Number of edges evaluated by wPA*SE and \wepase for varying \numthreads, with $\wh=\wi=1$ (top) and with $\wh=\wi=50$ (bottom).}
\label{epase/tab/nav3d_num_edges}
\end{table*}

Fig.~\ref{epase/fig/nav3d_plots}~(top) shows the average speedup achieved by wPA*SE and the baselines over wA* for varying \numthreads, for $\wh=\wi=1$ and $\wh=\wi=50$. The corresponding raw planning times are shown in Table~\ref{epase/tab/nav3d_epase_times}. The speedup achieved by PwA* saturates at the branching factor of the domain. This is expected since PwA* parallelizes the evaluation of the outgoing edges of a state being expanded. If \numthreads is greater than the branching factor \bfactor, $\numthreads-\bfactor$ threads remain unutilized.  For low \numthreads, the speedup achieved by \wepase matches that of wPA*SE. However, for high \numthreads, the speedup achieved by \wepase rapidly outpaces that of \wepase, especially for the inflated heuristic case. This is because \wepase is much more efficient than wPA*SE since it parallelizes edge evaluations instead of state expansions. This increased efficiency is more apparent with the availability of a larger computational budget in the form of a greater number of threads to allocate to evaluating edges. The speedup of wPA*SE reaches a peak and then rapidly deteriorates. This is also the case for \wepase w/o (improved) thread mgt. (described in the Method), even though the peak speedup of \wepaseold is higher than that of wPA*SE. However, the speedup of \wepase with the improved thread management strategy reaches a maximum and then saturates instead of degrading. This is due to the difference in the multithreading strategy employed by \wepase as explained in the Method.

Fig.~\ref{epase/fig/nav3d_plots}~(bottom) and Table~\ref{epase/tab/nav3d_num_edges} show that \wepase evaluates significantly fewer edges as compared to wPA*SE. With a greater number of threads, the difference is significant. This indicates that beyond parallelization of edge evaluations, the \weas formulation that \wepase uses has another advantage that if the heuristic is informative, \wepase evaluates fewer edges than wPA*SE, which contributes to the lower planning time of \wepase. The intuition behind this is that in wA*, the evaluation of the outgoing edges from a given state is tightly coupled with the expansion of the state because all the outgoing edges from a given state must be evaluated at the same time when the state is expanded. In \weas the evaluation of these edges is decoupled from each other since the search expands edges instead of states.

\subsection{Assembly Task}

\begin{figure}[]
    \centering
    \includegraphics[width=0.9\columnwidth]{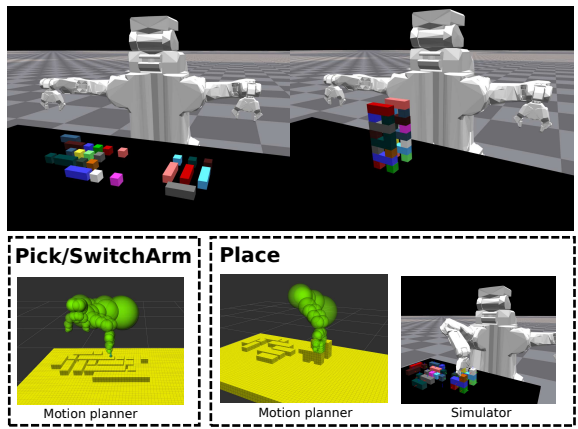}
    \caption{(Assembly) Top: The PR2 has to arrange a set of blocks on the table (left) into a given configuration (right). Bottom: It is equipped with \pick, \place and \swap controllers. The \pick controller uses the motion planner to reach a block. The \place controller uses the motion planner to place a block and simulates the outcome of releasing the block. The \swap controller uses the motion planner to move the active arm to a home position.}
    \label{epase/fig/assembly_problem_controllers}
\end{figure}

\begin{table}[]
\footnotesize
\centering
\begin{tabular}{ccccc}
\toprule
                           & wA*  & PwA* & wPA*SE & \wepase \\\midrule
\numthreads                & 1     & 25            & 10                            & 10    \\\midrule
Time (s)                   & 3010  & 1066          & 419                           & \textbf{301}                    \\
Speedup                    & 1     & 2.8           & 7.2                           & \textbf{10} \\\bottomrule
\end{tabular}
\caption{(Assembly) Average planning times and speedup over wA* for \wepase and the baselines.}
\label{epase/tab/assembly_times}
\end{table}

The second domain is a task and motion planning problem of assembling a set of blocks on a table into a given structure by a PR2, as shown in Fig.~\ref{epase/fig/assembly_problem_controllers}. This domain is similar to the one introduced in~\cite{mukherjee2022mplp}, but in this work, we enable the dual-arm functionality of the PR2. We assume full state observability of the 6D poses of the blocks and the robot's joint configuration. The goal is defined by the 6D poses of each block in the desired structure. The PR2 is equipped with \pick and \place controllers which are used as macro-actions in the high-level planning. In addition, there is a \swap controller which switches the active arm by moving the current active arm to a home position. All of these actions use a motion planner internally to compute collision-free trajectories in the workspace. Additionally, \place  has access to a simulator (NVIDIA Isaac Gym~\cite{makoviychuk2021isaac}) to simulate the outcome of placing a block at its desired pose. For example, if the planner tries to place a block at its final pose but has not placed the block underneath yet, the placed block will not be supported and the structure will not be stable. This would lead to an invalid successor during planning. We set a simulation timeout of  $\tsim=0.2$~\si{\second} to evaluate the outcome of placing a block. Considering the variability in the simulation speed and the overhead of communicating with the simulator, this results in a total wall time of less than $2$~\si{\second} for the simulation.  The motion planner has a timeout of $\tplan=60$~\si{\second} based on the wall time, and therefore that is the maximum time the motion planning can take. Successful \pick, \place and \swap actions have unit costs, and infinite otherwise. A \pick action on a block is successful if the motion planner finds a feasible trajectory to reach the block within \tplan. A \place action on a block is successful if the motion planner finds a feasible trajectory to place the block within \tplan and simulating the block placement results in the block coming to rest at the desired pose within \tsim. A \swap action is successful if the motion planner finds a feasible trajectory to the home position for the active arm within \tplan. The number of blocks that are not in their final desired pose is used as the admissible heuristic, with $\wh=\wi=5$. Table~\ref{epase/tab/assembly_times} shows planning times and speedup over wA* for \wepase and those of the baselines. We use 25 threads in the case of PwA* because that is the maximum branching factor in this domain. The numbers are averaged across 20 trials in each of which the blocks are arranged in random order on the table. Table~\ref{epase/tab/assembly_times} shows the average planning times and speedup over wA* of \wepase as compared to those of the lazy search baselines. \wepase achieves a 10x speedup over wA* and outperforms the baselines in this domain as well.

%% file: 08conclusion_future.tex
We presented an optimal parallel search algorithm \epase, that improves on PA*SE by parallelizing edge evaluations instead of state expansions. We also presented a sub-optimal variant \wepase and proved that it maintains bounded suboptimality guarantees. Our experiments showed that \wepase achieves an impressive reduction in planning time across two very different planning domains, which shows the generalizability of our conclusions. Empirically, we have observed \wepase to be a strict improvement over wPA*SE for domains with expensive to compute edges.  Even though we also test with a relatively large budget of threads, the performance improvement is significant even with a smaller budget of fewer than 10 threads, which is the case with typical mobile computers. Therefore in practice, we recommend using \wepase in a planning domain where 1) the computational bottleneck is edge evaluations and 2) successor states cannot be generated without evaluating edges and therefore parallelized lazy planning i.e. MPLP~\cite{mukherjee2022mplp} is not applicable.

MPLP~\cite{mukherjee2022mplp} and \epase use fundamentally different parallelization strategies. MPLP searches the graph lazily while evaluating edges in parallel, but relies on the assumption that states can be generated lazily without evaluating edges. On the other hand, \epase evaluates edges in a way that preserves optimality without the need for state (and edge) re-expansions but does not rely on lazy state generation. In domains where states can be generated lazily, however, the lazy state generation takes a non-trivial amount of time, these two different parallelization strategies can be combined. Both MPLP and \epase achieve a speedup at a certain number of threads, beyond which the speedup saturates. Therefore, utilizing both of them together will allow us to leverage more threads, which either of them on their own cannot.

%% file: 09acknowledgements.tex
This work was supported by the ARL-sponsored A2I2 program, contract W911NF-18-2-0218, and ONR grant N00014-18-1-2775.